\documentclass[12pt, a4paper]{article}

\usepackage[a4paper]{geometry}
\usepackage{pgfplots}
\pgfplotsset{
    x tick style={color=black},
    y tick style={color=black}
}

\usepackage{url}
\usepackage{listings}
\usepackage{amssymb}
\usepackage{graphicx}
\usepackage{algcompatible}
\usepackage{algorithm}
\usepackage{url}
\usepackage{rotating}
\usepackage{booktabs}
\usepackage{subfig}
\usepackage{amsmath}
\usepackage{bbm}

\renewcommand{\labelenumi}{(\alph{enumi})}
\renewcommand\theenumi\labelenumi

\usepackage[english]{babel}

\usepackage[utf8]{inputenc}
\usepackage{xspace}
\usepackage{amsmath,amsthm,amssymb,mathtools}
\usepackage{lmodern}

\usepackage[algo2e,ruled,vlined,linesnumbered]{algorithm2e}

\usepackage{xcolor}
\usepackage{tikz}
\usepackage{graphicx}
\usepackage{soul}

\allowdisplaybreaks[4]
\newtheorem{theorem}{Theorem}
\newtheorem{lemma}[theorem]{Lemma}

\newtheorem{definition}[theorem]{Definition}

\newcommand{\oea}{$(1 + 1)$~EA\xspace}

\newcommand{\oplea}{$(1+\lambda)$~EA\xspace}
\newcommand{\mpoea}{$(\mu+1)$~EA\xspace}
\newcommand{\oclea}{$(1,\lambda)$~EA\xspace}

\newcommand{\om}{\textsc{OneMax}\xspace}

\newcommand{\omw}{\textsc{OneMax$_w$}\xspace}
\newcommand{\omt}{\textsc{OneMax$_{(0,1^n)}$}\xspace}

\newcommand{\R}{\ensuremath{\mathbb{R}}}

\newcommand{\N}{\ensuremath{\mathbb{N}}} 
\newcommand{\Z}{\ensuremath{\mathbb{Z}}}

\let\originalleft\left
\let\originalright\right
\renewcommand{\left}{\mathopen{}\mathclose\bgroup\originalleft}
\renewcommand{\right}{\aftergroup\egroup\originalright}

\date{}
 
\begin{document}
\sloppy
\title{Theoretical Analyses of Evolutionary Algorithms on Time-Linkage \om with General Weights}

\author{Weijie Zheng\thanks{Weijie Zheng is with School of Computer Science and Technology, International Research Institute for Artificial Intelligence, Harbin Institute of Technology, Shenzhen, China, and was with Research Institute of Trustworthy Autonomous Systems (RITAS), Guangdong Provincial Key Laboratory of Brain-inspired Intelligent Computation, Department of Computer Science and Engineering, Southern University of Science and Technology, Shenzhen, China.}
\and Xin Yao\thanks{Corresponding author. Xin Yao is with Research Institute of Trustworthy Autonomous Systems (RITAS), Guangdong Provincial Key Laboratory of Brain-inspired Intelligent Computation, Department of Computer Science and Engineering, Southern University of Science and Technology, Shenzhen, China, and also with CERCIA, School of Computer Science, University of Birmingham, Birmingham, United Kingdom.}}

\maketitle
\begin{abstract}
Evolutionary computation has shown its superiority in dynamic optimization, but for the (dynamic) time-linkage problems, some theoretical studies have revealed the possible weakness of evolutionary computation. Since the theoretically analyzed time-linkage problem only considers the influence of an extremely strong negative time-linkage effect, it remains unclear whether the weakness also appears in problems with more general time-linkage effects. Besides, understanding in depth the relationship between time-linkage effect and algorithmic features is important to build up our knowledge of what algorithmic features are good at what kinds of problems. In this paper, we analyze the general time-linkage effect and consider the time-linkage \om with general weights whose absolute values reflect the strength and whose sign reflects the positive or negative influence. We prove that except for some small and positive time-linkage effects (that is, for weights $0$ and $1$), randomized local search (RLS) and \oea cannot converge to the global optimum with a positive probability. More precisely, for the negative time-linkage effect (for negative weights), both algorithms cannot efficiently reach the global optimum and the probability of failing to converge to the global optimum is at least $1-o(1)$. For the not so small positive time-linkage effect (positive weights greater than $1$), such a probability is at most $c+o(1)$ where $c$ is a constant strictly less than $1$.
\end{abstract}

\section{Introduction}\label{sec:int}
Many real-world applications have the time-linkage property, that is, the objective function relies on the current solution as well as the historical ones, or we say that the current solution has an impact on the future objectives. As a simple example, also given in~\cite{Bosman05}, the current vehicle routing solution is for serving the current existing orders, however, the quality of the service that it provides will influence the future orders that the company will receive, and thus influence the income of the company. Readers could see more than 30 applications in~\cite{Nguyen11}\footnote{Although the title of this literature is for continuous optimization, its survey contains both continuous and discrete time-linkage real-world applications.}.

Evolutionary computation is often a good choice for dynamic optimization problems~\cite{YangY13}.
For the dynamic time-linkage problems (solving them online),~\cite{Bosman05} pointed out the deceptive example for an arbitrarily bad optimization performance. The reason is that while the global optimum exists in the whole space, the optimization process for the time-linkage problems only searches in the subspaces, and then can stagnate in the local optima in the subspaces.

Although some literature worked on other aspects of the time-linkage problems, like~\cite{AllmendingerK13,BiswasDKP14,BuLZY17,ChengOMGSMY18,ZhangWYJY22}, the first rigorous theoretical results\footnote{For ``theoretical results'', we mean the convergence probability and the number of fitness evaluations to reach the optimum or a predefined goal.} on the time-linkage problem were only recently conducted in~\cite{ZhengCY21}. 
%
%
In~\cite{ZhengCY21}, they designed a time-linkage version on the widely-analyzed \om benchmark. The global optimum exists in the $\{0,1\}^{n+1}$ space, but the algorithm only searches in $\{0,1\}^n$ subspace. They proved that the randomized local search (RLS) algorithm and the \oea with $1-o(1)$ probability will get stuck to some of the local optima and cannot leave afterwards{, which implies that the time-linkage property can turn an easy problem to a hard one}. We note that it is quite different from the majority of the current theoretical results. There, the global optimum is in the same space that the algorithm will search, and the algorithm trivially converges if each individual is reachable from any state with a positive probability. Hence, the existing works focus more on runtime analysis than convergence. To the best of our knowledge, there are few results showing the non-convergence of evolutionary algorithms. For example, the non-convergence on multimodal functions can be caused by the one-bit mutation of RLS and its multi-objective counterpart, simple multi-objective optimizer (SEMO)~\cite{Giel03,QianTZ16,ZhengD23}. The non-convergence can also be caused by the survival selection in the NSGA-II with improper population size~\cite{ZhengLD22}, or caused by the wrong choice of the reference point for the multi-objective evolutionary algorithms with diversity-based parent selection~\cite{OsunaGNS20}.  The non-convergence in binary differential evolution on the artificial function is caused by the mutation~\cite{DoerrZ20}, and the non-convergence in the estimation of distribution algorithms without artificial margins stems from the genetic drift~\cite{KrejcaW20}. The non-convergence of \oea also appears in the noisy or dynamic environments~\cite{KotzingM12}.

In addition to the above interesting difference that the time-linkage problems bring to the evolutionary algorithms,~\cite{ZhengZCY21} showed that the time-linkage problem in~\cite{ZhengCY21} also provides the situation that the non-elitist algorithms can be theoretically beneficial. In~\cite{ZhengZCY21}, they proved that comparing with \oplea getting stuck with probability $1-o(1)$, its counterpart \oclea and any non-elitist algorithm that can accept the inferior point, reach the global optimum of their discussed time-linkage \om problem with probability $1$. Note that theoretically positive support for the non-elitist algorithm is not much~\cite{Doerr20gecco} {(Reader may refer to~\cite{DangEL21} to see the recent theoretical support of the non-elitist EAs escaping local optima, but we omit more details about this paper and other theoretical discussions about the non-elitism for the concentration of this paper)}.

In summary, the time-linkage \om in~\cite{ZhengCY21} provides interesting special situations for evolutionary algorithms. However, this time-linkage \om only considers the influence of an extremely strong negative time-linkage effect, and thus limits the general understanding on the behavior of the evolutionary algorithms solving the time-linkage problem, and limits the possible guidance for the practical usage. Besides, it is equally important to understand in depth the relationship between problem characteristics and algorithmic features so that we can build up our knowledge of what algorithmic features are good at what kinds of problems. In the time-linkage \om~\cite{ZhengCY21,ZhengZCY21}, 
the time-linkage effect is represented by the weight of the historical solution. The absolute value of the weight reflects the strength and the sign reflects the positive or negative influence. The analyzed time-linkage \om only considers time-linkage effect on the first dimension, and uses
 an extreme negative weight (extremely strong negative strength) of $-n$ (where $n$ is the problem size for the current time step and is also the maximum value that the current time step can contribute to the objective function). 
 Other types of time linkage have not been analyzed. To the best of our knowledge,~\cite{ZhengCY21} and~\cite{ZhengZCY21} are the only studies that theoretically discussed the behavior of the evolutionary algorithms on the time-linkage problems, and no literature exists analyzing how evolutionary algorithms cope with the time-linkage problem with general weight. {We note that as pointed out in~\cite{HeCY15}, the hardness of the optimization problem is an important topic in the EA community. The time-linkage property with an extreme effect indicates the possible difficulty for the EAs. Thus, the interesting question of how the difficulty of the problem (w.r.t. the EAs) changes for different time-linkage weights remains unsolved.} A deeper understanding between the time-linkage property and the algorithm features is still missing.

In this paper, we conduct such a step towards understanding how evolutionary algorithms tackle the time-linkage problem{, and towards understanding the hardness of the problem,} for different time-linkage strengths and influences. Instead of the extremely strong negative effect (weight of $-n$), we generalize the time-linkage \om in~\cite{ZhengCY21} by regarding a weight $w$ that could be any integer value, and call it \omw. We consider the behavior of RLS and \oea, and show their behavior changes for different strengths and influences. We prove that except for some small and positive time-linkage effects (that is, for weights $0$ and $1$), the RLS and \oea get stuck and cannot reach the global optimum afterwards with a positive probability. Note that the different magnitudes of the failure probability have different indications for the practical guidance. Taking the restart strategy as an example, for the $1-o(1)$ level of failure probability, $\omega(1)$ number of restarts are expected to witness a success, which can be inefficient; for a constant level of failure probability, we need a constant number of restarts; and for the $o(1)$ level, the failure rarely happens and we might reach the global optimum in one run and not need to restart for sufficiently large problem size.  For this, we need more precise results for the failure probability. We prove that the failure probability is $1-o(1)$ for the negative time-linkage influence (negative weight). For not so small positive time-linkage influence (the weight $w>1$), the failure probability is at least $\Theta(\min\{w/n,c\})$ for a certain constant $c\in(0,1)$, but at most $1-\left(1-2\exp\left(-n/24\right)\right)\left(e+2-2/n\right)/(4e^{3e/4+1})$ (and the probability to reach the global optimum is at least $\left(1-2\exp\left(-n/24\right)\right)\left(e+2-2/n\right)/(4e^{3e/4+1})$). Besides, for non-negative time-linkage effect, we also prove that conditional on an event that happens with at least a constant probability, the RLS and \oea reach the global optimum of \omw in $O(n\log n)$, which is the same asymptotic complexity for the \om that does not have the time-linkage property. 

The remainder of this paper is organized as follows. Section~\ref{sec:pre} includes the preliminaries. Our generalized time-linkage \om is introduced in Section~\ref{sec:omw}. Sections~\ref{sec:negw} and~\ref{sec:nonnegw} respectively show our theoretical results for the time-linkage \om with negative and non-negative weights. A further discussion is given in Section~\ref{sec:exp}. Section~\ref{sec:con} concludes our work.

\section{Preliminaries}
\label{sec:pre}
{As discussed in Section~\ref{sec:int}, this paper will analyze a more general time-linkage problem, to see the behavior of EAs on different time-linkage strengths and influences, that is, to systematically see whether the time-linkage property changes the difficulty of the problems w.r.t. the EAs. As preliminaries, in this section, we will briefly introduce the easy problem \om, the time-linkage property,  the analyzed time-linkage \om (\om$_{(0,1^n)}$) that turns the \om becoming a more difficult problem, and some existing runtime results for the \om$_{(0,1^n)}$.}
{\subsection{\om}}\label{ssec:om}
{\om is one of the widely analyzed benchmarks in the evolutionary theory community. For a bitstring $x\in\{0,1\}^n$, the \om fitness is the number of ones in $x$. Already in 1990s and early 2000s, researchers~\cite{Muhlenbein92,GarnierKS99} proved that the \oea optimizes the \om with problem size $n$ in expected runtime of $O(n \log n)$. Doerr, Johannsen, and Winzen~\cite{DoerrJW10} proved that w.r.t. the expected runtime of \oea with mutation probability of $1/n$, \om is the easiest function among all functions with a unique global optimum. Witt~\cite{Witt13} generalized the mutation probability from $1/n$ to any $p\le 1/2$, and proved that any mutation based EA with population size $\mu$ with $p\le 1/2$ on any function with a unique global optimum will have at least as large as the runtime of $(1+1)$~EA$_{\mu}$ on \om. Besides, in the systematic analysis of the easiest and hardest functions, He, Chen, and Yao~\cite{HeCY15} also proved \om is the easiest benchmark w.r.t. the \oea among all linear functions, as an example of applying their theorem. In summary, \om is an easy problem.}

\subsection{Time-Linkage Problem{s} and Solving Modes}
The time-linkage problem{s}, introduced into the evolutionary community by~\cite{Bosman05}, is the problem{s} where the current solution (decision)  impacts the future objective. The formal description of the (discrete) time-linkage pseudo-Boolean problem ${h: \{0,1\}^n\times\dots\times\{0,1\}^n \rightarrow \R}$~\cite{ZhengCY21} is given as follows
\begin{align}
h(x^{ {t_s}},\dots,x^{ {t_s+\ell}})=\sum_{t=0}^{\ell} h_t(x^{ {t_s+t}}; x^{ {t_s}},\dots,x^{ {t_s+t-1}}),
\label{eq:h}
\end{align}
where {$h_t()$ is the function slice at time $t$,} $x^{t_s},\dots,x^{t_s+\ell}$ are consecutive solutions and $t_s$ is the starting time{, and ``;'' is used to separate the current solution and its historical ones for time-linkage effect}.  Usually, $h$ dynamically changes when the decision at the time $t_s+\ell$ (solution $x^{t_s+\ell}$) is made, that is, when $\ell$ increases.

Generally, the time-linkage problem can be solved in two modes, the \emph{offline} mode and the \emph{online} mode~\cite{Bosman05,ZhengCY21}. Assume that we already have the decision (solution) sequence till time $t'$, $x^{ {t_s}},\dots,x^{ {t'-1}}$. If the evaluation of all decision sequence with different $x'^{ {t_s}},\dots,x'^{ {t'-1}}$ is possible, then the problem can be solved in \emph{the offline mode}. For example, this mode can happen in a situation where the changing states after each decision at each time are deterministic (but unknown). In this case, the ``time'' is not the real-world time, but merely the order of an item in the sequence. {The ``decision'' sequence is not the already implemented sequence but just a pre-``decision'' (trial) sequence, and it (including the historical time-linkage part and the current part) always has a fixed length (say a fixed length of $\ell$ for (\ref{eq:h})). One can evaluate the fitness of multiple trial sequences, and then finally select the best one among all trial sequences as the unique one for the implementation.} For the offline mode, there are different solving strategies. For example, one can regard (\ref{eq:h}) as a static optimization problem with $\{0,1\}^{n(\ell+1)}$ as the input space {(that is, it contains no historical time-linkage part)} and use the optimization algorithms like evolutionary algorithms to solve it. For this solving strategy, no time-linkage characteristics are analyzed and thus this solving strategy is not interesting in this work. Instead, we divide $x^{ {t_s}},\dots,x^{ {t_s+\ell}}$ into $x^{ {t_s}},\dots,x^{ {t_s+\ell'}}$ {(historical part of length $\ell'+1$)} and $x^{ {t_s+\ell'+1}},\dots,x^{ {t_s+\ell}}$. We can only optimize the latter part and store the former part for fitness evaluations. In this solving strategy, we can explore how evolutionary algorithms react to the time-linkage property. 

If we can not evaluate any decision sequence different from the already existing $x^{ {t_s}},\dots,x^{ {t'-1}}$, the problem needs to be solved \emph{online}. That is, what decisions that have been made previously cannot be changed in the current and future time, the ``time'' is the real-world time, the problem needs to be solved as time goes by, and we can only evaluate the quality of the possible current solution (before decision) at time $t'$ with historical decision sequence $x^{ {t_s}},\dots,x^{ {t'-1}}$. There are different solving strategies, like only optimizing the present with or without prediction. Following~\cite{ZhengCY21}, we consider optimizing the present problem instance without prediction for the online mode.

\subsection{\omt}
The \omt function is proposed in~\cite{ZhengCY21}, and is the only time-linkage benchmark function for the theoretical analyses~\cite{ZhengZCY21}, to the best of our knowledge. It consists of two components. The component for the current time step adopts the well-analyzed \om problem in the evolutionary theory community, that is, the number of ones in the binary bit-string (solution). The other component takes the last time step and only the first dimension for the time-linkage effect. 
It considers a negative influence of the time-linkage part via setting the time-linkage dimension with a negative weight. An extreme time-linkage strength is adopted via setting the absolute value of the weight as the problem size $n$ to ease the theoretical analysis.
The formal definition of \omt is shown in the following.
\begin{definition}~\cite{ZhengCY21}
The \omt function $f: \{0,1\}^n \times \{0,1\}^n \rightarrow \R$ is defined by
\begin{align*}
f(x^{t-1},x^{t})=\sum_{i=1}^n x_{i}^{t}-nx_1^{t-1},
\end{align*}
where $x^{t-1}=(x_1^{t-1},\dots,x_n^{t-1})\in\{0,1\}^n$ and $x^{t}=(x_1^{t},\dots,x_n^{t})\in\{0,1\}^n$ are two consecutive solutions. 
\end{definition}

In terms of the solving modes in the above subsection, the \omt can be obviously solved in the offline mode. However, it is not interesting to regard \omt as the time-linkage problem to be solved online as there are only two time steps. As pointed in~\cite{ZhengCY21}, \omt can be the component (at the current and intermediate previous time) of the following problem 
\begin{align}
h(x^{ {0}},\dots,x^{ {t}}) = \sum_{ {\tau=2}}^{ {t}} e^{ {-t+\tau-1}} x^{ {\tau-2}}_1 - n x^{ {t-1}}_1 + \sum_{i=1}^n x^{ {t}}_i
\label{eq:sumom}
\end{align}
for the consecutive solutions $x^{0},\dots,x^{t}$, where ${x^{ {\tau}}=(x_1^{ {\tau}},\dots,x_n^{ {\tau}}) \in \{0,1\}^n}$ for ${ {\tau=0,1,\dots,t}}$. Obviously, (\ref{eq:sumom}) dynamically changes when a new decision is made (also note that the coefficient of a fixed $x_1^{t'} (t'\ge 2)$ changes from $1$ to $-n,e^{-1}, e^{-2}, \dots$). When taking the solving strategy of optimizing the present without prediction for the online mode, it is identical to solving the \omt in the offline mode. Therefore,~\cite{ZhengCY21} calls solving (\ref{eq:sumom}) online {(like w.r.t. \oea or other algorithms with only one parent)} by solving \omt online for simplicity, and theoretically discusses the behaviors of several evolutionary algorithms on \omt without distinguishing the online or offline mode.

For the maximization of \omt, the global optimum is $(x_1^{t'-1},x^{t'})=(0,1^n)$ for some $t'\in \Z_{\ge 0}$, and the function value is $n$.

\subsection{RLS and \oea on \omt}
{As discussed in Section~\ref{ssec:om}, \om is regarded as an easy problem. However, the existing result for the  \om with extreme time-linkage effect, \omt in~\cite{ZhengCY21}, shows the possible difficulty of the EAs, especially \oea. Here, we give a brief introduction here.}

To optimize the \omt that involves two time steps,~\cite{ZhengCY21} slightly modified the RLS and \oea that are usually for the problem without time-linkage property. The general framework is the same as for the one without time-linkage property, containing the mutation and selection, and the only modification is that the current solution needs the previous solution and the generated offspring needs its parent solution for the fitness evaluations. Their modified algorithm is in Algorithm~\ref{alg:rlsoea}\footnote{Compared to the original algorithm description in~\cite{ZhengCY21}, we introduce the notation of $t$ for mathematically more precise statements.}. 
We note that $t$ in this algorithm represents the ``time'' that a decision is made, and $g$ is the generation number, which is also a counter for how many times the function is evaluated for both algorithms. The offspring $\tilde{X}^{(g)}$ is a trial from the current solution $X^t$ (which is also $X^{(g)}$), and is accepted if {and} only if $\tilde{X}^{(g)}$ is at least as good as $X^t$. Once $\tilde{X}^{(g)}$ is accepted, the decision in time $t+1$ is made. See lines 7-10. 
In this paper, we will still call ``time-linkage RLS and \oea'' as ``RLS and \oea'' for simplicity, and analyze them.

\begin{algorithm}[tb]
\caption{Time-linkage RLS and \oea to maximize $f$ requiring the evaluation of two consecutive time steps}
\label{alg:rlsoea}
\begin{algorithmic}[1] 
\STATE Generate $X^{0}=(X^0_1,\dots,X^0_n)$ and $X^1=(X^1_1,\dots,X^1_n)$ from $\{0,1\}^n$ uniformly at random; 
\STATE {$t=1$};
\FOR{$g=0,1,2,\dots$}
{
\STATE Let $X^{(g)}:=X^t$. Generate $\tilde{X}^{(g)}$ by\\
{\textit{\% Mutation}}\\
-- \textbf{RLS}: choose $i\in[1..n]$ uniformly at random and flip the $i$-th bit of $X^{(g)}$; \\
-- \textbf{\oea}: Independently flip each bit of $X^{(g)}$ with probability of $1/n$;\\
{\textit{\% Selection: Replacement happens (Decision at time $t+1$ is made) when $\tilde{X}^{(g)}$ is at least as good as its parent $X^{(g)}=X^t$}}
\IF{$f(X^{(g)},\tilde{X}^{(g)}) \ge f(X^{t-1},X^t)$}
\STATE {$X^{t+1}=\tilde{X}^{(g)}, t=t+1$}
\ENDIF
}
\ENDFOR
\end{algorithmic}
\end{algorithm}

They proved that with high probability RLS and \oea cannot reach the global optimum, see the following theorem. We note that they also provided the theoretical analysis on \mpoea, but we omit it since our paper only focuses on the RLS and \oea.  
\begin{theorem}~\cite[Theorem~1]{ZhengCY21}
For the $n$-dimensional $(n\ge 6)$ \omt function, with probability at least $1-(n+1)\exp{\left(-n^{1/3}/e\right)}-(e+1)/n^{1/3}$, RLS and \oea cannot reach the global optimum. 
\label{thm:omt}
\end{theorem}

In our proof in Sections~\ref{sec:negw} and~\ref{sec:nonnegw}, we will use some results from~\cite{ZhengCY21}, and hence we give them in the following lemma. Besides, in this paper, we use $|x|$ to denote the number of ones in the bit string $x$. 
\begin{lemma}~\cite[Lemma 2 and proof in Lemma 4]{ZhengCY21}
Given any $X\in\{0,1\}^n\setminus\{1^n\}$. Let $Y$ be generated by the bit-wise mutation with rate $1/n$ or one-bit mutation on $X$, and $a$ be the number of zeros in $X$. Then the following two facts hold.
\begin{itemize}
\item[(a)]\label{lem:fact1} $\Pr[|Y|-|X|=1 \mid |Y| > |X|] > 1-ea/n$.
\item[(b)]\label{lem:fact2} For any certain bit position $i$ in $X$ with value $0$, $\Pr[Y_i=1 \mid |Y|-|X|=1] \ge 1/a$.
\end{itemize}
\label{lem:2facts}
\end{lemma}

\section{Time-Linkage \om with General Weights}
\label{sec:omw}
In Section~\ref{sec:pre}, we see {the extreme settings of} the time-linkage benchmark \omt that existing literature~\cite{ZhengCY21,ZhengZCY21} analyzed{. \omt} only considers the time-linkage effect on the first dimension value of one-time step, and uses an extreme time-linkage strength with a negative influence (weight $-n$). We acknowledge that for the first time-linkage benchmark such a simple and extreme design is reasonable and beneficial for conducting the first rigorous analysis. {Their results indicate that the time-linkage property has the possibility to turn the easy \om to be a hard one w.r.t. some EAs.} However, it is questionable whether the theoretical findings obtained for such an extreme setting 
{also hold for general time-linkage property (that is, whether the time-linkage property with different strengths and influences always make \om harder),}
and it is natural that such an extreme setting is far from the practical application, so that their guidance for the practical usage is quite limited. Besides, it is equally important to understand in depth the relationship between problem characteristics and algorithmic features so that we can build up our knowledge of what algorithmic features are good at what kinds of problems. For pure theoretical curiosity of essentially understanding the overall performance of the algorithms on the general time-linkage problems, to deeply understand the relationship between the time-linkage property and the algorithm feature, and to approach the ultimate aim of theoretical guidance for practical usage, this work will consider more general time-linkage problems than \omt.

As discussed above, two aspects of the \omt need to be generalized, the number of dimensions for the time-linkage effect and the extreme weight of the time-linkage first dimension in the previous time step. This paper only focuses on the generalization of the weight of the time-linkage first dimension in the previous time step, and leaves the generalization of the number of time-linkage dimensions as interesting future work. Different from the weight of $-n$ for the previous first dimension in the \omt function, we consider the weight $w$ that can be any integer. That is, with any given $w\in\Z$, we consider the time-linkage function 
\begin{equation}
f(x^{t-1},x^t)=\sum_{i=1}^n x_i^t+w x_1^{t-1},
\label{eq:omw}
\end{equation}
where $x^{t-1}=(x_1^{t-1},\dots,x_n^{t-1})\in\{0,1\}^n$ and $x^{t}=(x_1^{t},\dots,x_n^{t})\in\{0,1\}^n$ are two consecutive solutions. 
In the remainder of this paper, we will call it \omw. Note that the absolute value $|w|$ reflects the time-linkage strength, and the sign reflects the negative or positive influence. Hence, different values of $w$ will indicate different time-linkage strengths and influences.

Similar to \omt in~\cite{ZhengCY21} and discussed in Section~\ref{sec:pre}, we only discuss the offline solving strategy {in this paper}. We also note that analogous to (\ref{eq:sumom}) for \omt, {the online setting (w.r.t. the algorithm with only one parent) of} solving the dynamic time-linkage
\begin{equation*}
h(x^{ {0}},\dots,x^{ {t}}) = \sum_{ {\tau=2}}^{ {t}} e^{ {-t+\tau-1}} x^{ {\tau-2}}_1 +w x^{ {t-1}}_1 + \sum_{i=1}^n x^{ {t}}_i
\label{eq:sumomw}
\end{equation*}
with the strategy of optimizing the present without prediction can {also} be transferred to solving \omw offline.

For maximization, when $w <0$, the global optimum is $(x_1^{t'-1},x^{t'})=(0,1^n)$ for some $t'\in \Z_{\ge 0}$, which is the same as the \omt function; when $w=0$, the global optimum is $x^{t'}=1^n$ with no restrictions on $x^{t'-1}$ for some $t'\in \Z_{\ge 0}$, which is the same as the \om function without the time-linkage property; when $w>0$, it is $(x_1^{t'-1},x^{t'})=(1,1^n)$ for some $t'\in \Z_{\ge 0}$. For all cases, the maximum function value is $n$.

Despite other kinds of general time-linkage problems, we focus on (\ref{eq:omw}) as it follows the same line as the only existing theoretical benchmark \omt.

\section{RLS and \oea on $\omw_{\in \Z_{<0}}$}
\label{sec:negw}
Intuitively, the current first bit value prefers the value of $1$, which will not be preferred once the current solution is accepted and turned to the ``previous'' for the future new solution since $w<0$ results in the better objective value for the previous first bit value of $0$ than $1$.~\cite{ZhengCY21} has proved the $1-o(1)$ probability of the non-convergence to the global optimum for the RLS and \oea when $w=-n$. In this section, we will discuss whether the high non-convergence probability of the RLS and \oea still holds on the \omw with any $w\in\Z_{<0}$. 
 
\subsection{Global Optimum and Stagnation Cases}
Similar to~\cite{ZhengCY21}, we analyze the RLS and \oea in Algorithm~\ref{alg:rlsoea}, and say the algorithm reaches the global optimum if there is a certain $t'>0$ such that $X^{t'}=1^n$ with stored $X^{t'-1}_1=0$. We note that the global optimum for $w\in \Z_{<0}$ or more generally $w < 0$ is the same as the one in~\cite{ZhengCY21}, that is, $w=-n$.

Note that the algorithm searches in the $n$-dimensional space but the global optimum exists in the $(n+1)$-dimensional space, and thus the convergence is not trivial for the RLS and \oea.~\cite{ZhengCY21} has pointed out two kinds of stagnation cases. Similarly, for the \omw with $w\in\Z_{<0}$, there are also two possible  stagnation cases for RLS and \oea, but with the first case (Event I) dependent on $w$, see the following lemma.
\begin{lemma}
Let $w\in \Z_{<0}$. Consider using the RLS / \oea to optimize the $n$-dimensional \omw function. Let $X^0, X^1, \dots$ denote the solution sequence. Let
\begin{itemize}
\item \textsl{Event I (only for $w\in \Z_{<-1}$):} 
\textsl{(a)} For \oea: there is a $t_0 \in \N$ such that $(X_1^{t_0-1}, X_1^{t_0})=(0,1)$, $X^{t_0}\ne 1^n$, and $|X_{[2..n]}^{t_0}| \in [w+n..n-2]$; 
\textsl{(b)} For RLS: there is a $t_0 \in \N$ such that $(X_1^{t_0-1}, X_1^{t_0})=(0,1)$, $X^{t_0}\ne 1^n$.
\item \textsl{Event II:} there is a $t_0 \in \N$ such that $(X_1^{t_0-1},X^{t_0})=(1,1^n)$.
\end{itemize}
If Event I or Event II happens at a certain time, then RLS / \oea cannot find the global optimum of the \omw in any arbitrary long runtime afterwards.
\label{lem:stagneg}
\end{lemma}
\begin{proof}
For \oea, if Event I happens (say at generation $g_0$), then $X^{(g_0)}=X^{t_0}$ will have the fitness of at least $w+n+1$. However, from $X_1^{(g_0)}=1$ we know that any offspring $\tilde{X}^{(g_0)}$ will have the fitness value of at most $w+n < w+n+1$. Hence, $\tilde{X}^{(g_0)}$ cannot replace $X^{(g_0)}$, and the stagnation happens.

For RLS, if Event I happens (say at generation $g_0$), then from $X_1^{(g_0)}=X_1^{t_0}=1$ we know that any offspring $\tilde{X}^{(g_0)}$ will have the fitness value of at most $w+|X^{(g_0)}| +1 < |X^{(g_0)}|$, where the inequality uses $w<-1$.  Hence, $\tilde{X}^{(g_0)}$ cannot replace $X^{(g_0)}$, and the stagnation happens.

If Event II happens (say at generation $g_0$), then $X^{(g_0)}=X^{t_0}$ will have the fitness of $w+n$. From $X_1^{(g_0)}=1$ we know that any offspring $\tilde{X}^{(g_0)}$ will have the fitness value of at most $w+n$, and take the fitness value of $w+n$ only when $\tilde{X}^{(g_0)}=1^n=X^{t_0}$. Hence, even when replacement happens we have $(X_1^{t_0},X^{t_0+1})=(1,1^n)$. The stagnation happens.
\end{proof}

{Once Event II happens}, any accepted offspring $\tilde{X}^{(g)}$ can only be $1^n$, which will again result in the occurrence of Event II. More intuitively, Event II means that the solution has fallen into the optimum $(1,1^n)$ of the subspace $\{1\}\times\{0,1\}^n$ (the whole $(n+1)-$dimensional space is $\{0,1\}\times\{0,1\}^n$), and any generated offspring is still in this subspace and cannot defeat its parent $(1,1^n)$, and thus cannot further jump to the subspace $\{0\}\times\{0,1\}^n$ where the global optimum $(0,1^n)$ is located. The stagnation of Event I is from the fact that any generated offspring will have the first bit pattern of $(1,*)$ with $*\in\{0,1\}$, thus a lower fitness than its parent when the parent satisfies the condition in Event I. Therefore the offspring cannot replace its parent to the next generation. More intuitively, Event I means that the solution (not the global optimum) is in some states that any offspring will be in the subspace $\{1\}\times\{0,1\}^n$ and the maximal fitness of this subspace is still strictly less than its parent, and thus the replacement cannot happen and the stagnation occurs.

\subsection{Non-global-convergence When $w\in \Z_{\le -n}, n>0$}
From Lemma~\ref{lem:stagneg}, we know that when $w\in \Z_{\le -n}$, Event I becomes that there is a $t_0 \in \N$ such that $(X_1^{t_0-1}, X_1^{t_0})=(0,1)$, since for any $x\in\{0,1\}^{n-1} \setminus \{1^{n-1}\}, |x| \in [0..n-2] \subseteq [w+n..n-2]$. That is, the stagnation cases for all $w\in \Z_{\le -n}$ are identical. 
Besides, in the following, we will show that for $w\in \Z_{< -n}$, the behavior of RLS and \oea is identical for \omt ($w=-n$), via showing that the selection of the RLS and \oea will keep the same individuals. That is, {the individual survives} during the selection {w.r.t.} $w=-n$ also survives for $w<-n$, and {the individual survives} for $w<-n$ also survives for $w=-n$, see the following lemma.
\begin{lemma}
Let $f_w$ denote the \omw function. For any $x,y,z \in \{0,1\}^n$, $f_{-n}(x,y) \le f_{-n}(y,z) \Leftrightarrow f_{<-n}(x,y) \le f_{<-n}(y,z)$, where $f_{<-n}$ means $f_{w}$ with any specific $w<-n$.
\label{lem:nonzero}
\end{lemma}
\begin{proof}
$(x_1,y_1)$ can only take its possible value from $(0,0),(0,1),(1,0),$ and $(1,1)$. We first consider $x_1=y_1$, that is, $(x_1,y_1)\in\{(0,0),(1,1)\}$. 
Noting that $f_w(x,y) \le f_w(y,z) \Leftrightarrow wx_1+|y| \le wy_1+|z| \Leftrightarrow w(x_1-y_1) \le |z| - |y|$, we know in this case, the above is identical to $0 \le |z| -|y|$, which is independent of $w$, thus $f_{-n}(x,y) \le f_{-n}(y,z) \Leftrightarrow f_{<-n}(x,y) \le f_{<-n}(y,z)$ holds trivially. 

If $(x_1,y_1)=(0,1)$, for any $w\in \Z_{\le -n}$ if $f_w(x,y) \le f_w(y,z)$ then $ |y| \le w+ |z| \le -n + |z| \le 0$, which is contrary to $|y| \ge y_1 =1$. Hence, $(x_1,y_1)=(0,1)$ cannot happen for $f_w(x,y) \le f_w(y,z)$ with $w\in\Z_{\le -n}$.

If $(x_1,y_1)=(1,0)$, for $w\in \Z_{\le -n}$ we know $f_w(x,y) \le f_w(y,z) \Leftrightarrow w+|y| \le |z| $. Since $w+|y| \le -n + |y| \le 0 \le |z|$ trivially holds for any $w\in \Z_{\le -n}$, we know $f_w(x,y) \le f_w(y,z)$ holds independent of $w$.\end{proof}

Hence, from {Lemmata}~\ref{lem:stagneg} and~\ref{lem:nonzero}, we know that the process for $w\in\Z_{\le -n}$ is identical to the \omt ($w=-n$), hence, the convergence results of the \omt also holds for $w\in\Z_{\le -n}$. To obtain the convergence results, we first show that conditional on the initial $|X^1| > n/4$, with a high probability, the global optimum cannot be reached before the individual decreases its number of zeros bits to $n^{1/3}$. 
\begin{lemma}
{Let $b,c\in(0,1/2)$ be the given constants. Let $n\ge\max\{2^{1/c},4^e\}$ be large enough such that $e(\log_2 n)^2\le n^c$, and $w\in[-n..-1]$}. Consider using RLS / \oea  to optimize the $n$-dimensional \omw function. Assume that $|X^1| > b n$. Then with probability of {$1-17\log_2 n/n$}, the global optimum cannot be found before the number of zeros of the current solution decreases below $n^{c}$.
\label{lem:startingassump}
\end{lemma}
\begin{proof}
Since RLS only changes one bit for one generation, it is not difficult to see that the probability of reaching the global optimum before the number of zeros of the current solution decreases below $n^{c}$ is 0. The following only discusses the \oea.

It is easy to see that if one of Event I and II happens before the number of zeros of the current solution decreases below $n^c$, then the process will get stuck and the global optimum cannot be reached afterwards, which supports this claim. Hence, in the following, we assume that any stagnation case will not happen in the process before the number of zeros in the current solution decreases below $n^c$.

Let $a$ denote the number of zeros in the current individual. Then the probability that $a$ decreases to 0 in one generation is
$
(1/n^a)\left(1-1/n\right)^{n-a} \le 1/n^a.
$
{Let event $A$ denote that the global optimum is reached in one generation.}
Since the global optimum requires $a=0$ as well as the stored first bit value of $0$, we know {$\Pr[A]$,} the probability of reaching the global optimum in one generation, is at most
\begin{align}
\frac{1}{n^a}.
\label{eq:onesuc}
\end{align} 

{Let event $B$ denote that $a$ decreases in one generation but does not decrease to $0$, and let $B'$ be the event that $(1,0)$ first bit pattern is generated and survives. Note that $B'$ includes the case when $a$ does not change in one generation. As $B$ includes the case when $a$ decreases by $1$, we have
\begin{align*}
\Pr[B] \ge \frac{a}{n}\left(1-\frac 1n\right)^{n-1} \ge \frac{a}{en}.
\end{align*}
To estimate $\Pr[B']$, we know that the first bit must flip from $1$ to $0$ and at least one of $0$s in the current individual must be flipped. Hence, we have
\begin{align*}
\Pr[B'] \le \frac1n\frac{a}{n}=\frac{a}{n^2}.
\end{align*}}
{Then 
\begin{align}
\Pr[A \mid A\cup B \cup B'] = \frac{\Pr[A]}{\Pr[A\cup B\cup B']} \le \frac{\Pr[A]}{\Pr[B]} \le \frac{e}{an^{a-1}} \le \frac{e}{n^{n^c+c}},
\label{eq:AB}
\end{align}
where the last inequality uses $a\ge n^c+1$, and
\begin{align}
\Pr[B' \mid B\cup B'] \le \frac{\Pr[B']}{\Pr[B]} \le \frac en.
\label{eq:Bp}
\end{align}}

{Via (\ref{eq:AB}), we know that the probability that $B\cup B'$ (with $a \ge n^c+1$) happens $n$ times (if possible) before $A$ happens once is at least
\begin{align}
\left(1-\frac{e}{n^{n^c+c}}\right)^n \ge 1-\frac{en}{n^{n^c+c}} =1-\frac{e}{n^{n^c+c-1}}.
\label{eq:pBBp}
\end{align}}

{Consider the process that only $B$ or $B'$ happens, and let $Y$ be the number of times that $B'$ happens when $B$ or $B'$ occurs $n$ times. Via (\ref{eq:Bp}), we know that $Y$ is stochastically dominated by a random variable that obeys the binomial distribution with the success probability of $e/n$. Then via the Chernoff bound (See\cite[(1.10.2)]{Doerr20}), we have
\begin{align*}
\Pr[Y \ge e\log_2 n]\le \left(\frac{e^{\log_2 n-1}}{(\log_2 n)^{\log_2 n}}\right)^{e} =\left(\frac1e\left(\frac{e}{\log_2 n}\right)^{\log_2 n}\right)^{e} \le \left(\frac1e\frac{1}{2^{\log_2 n}}\right)^e =\frac{1}{(en)^e},
\end{align*}
where the last inequality uses $\log_2 n \ge 2e$ for $n \ge 4^e$. Then with probability at least 
\begin{align}
1-\frac{1}{(en)^e},
\label{eq:pBp}
\end{align}
$B'$ happens at most $e\log_2 n - 1$ times when $B$ or $B'$ occurs $n$ times.}

{Now we consider the situations in the next generation after $B'$ happens. It is not difficult to see that $a$ increases only when the current individual has $(1,0)$ first bit pattern (that is, in the generation right after the occurrence of $B'$) and generates an offspring with more zeros. Let $B'_+$ be the event that $a$ increases by at least $\log_2 n$ or that $a$ decreases to $0$, and $B'_{\le}$ be the event that $a$ increases by at most $\log_2 n -1$ but does not decrease to $0$. As $B'_+$ requires that at least $\log_2 n$ number of $1$s flip or that $a$ decreases to $0$, we have
\begin{align*}
\Pr[B'_+] &\le{} \frac{\binom{n-a}{\log_2 n}}{n^{\log_2 n}} + \frac{1}{n^a} \le \frac{\left(\frac{e(n-a)}{\log_2 n}\right)^{\log_2 n}}{n^{\log_2 n}} + \frac{1}{n} =\left(\frac{e}{\log_2 n} \frac{n-a}{n}\right)^{\log_2 n} + \frac{1}{n} \\
&\le{} \left(\frac{e}{\log_2 n}\right)^{\log_2 n} + \frac{1}{n} \le \left(\frac12\right)^{\log_2 n}+ \frac{1}{n} =\frac 2n,
\end{align*}
where the last inequality uses $\log_2 n \ge 2e$ for $n \ge 4^{e}$. For $B'_{\le}$, we pessimistically consider the case that $a$ does not change and have
\begin{align*}
\Pr\left[B'_{\le}\right] \ge \left(1-\frac 1n\right)^n\ge \frac{1}e \left(1-\frac 1n\right) \ge \frac{9}{10e},
\end{align*}
where the last inequality uses $1-1/n \ge 9/10$ for $n \ge 4^e$. Then
\begin{align*}
\frac{\Pr[B'_+]}{\Pr[B'_+\cup B'_{\le}]} \le \frac{\Pr[B'_+]}{\Pr[B'_{\le}]} \le \frac{20e}{9n}.
\end{align*}
Hence, the probability that $B'_+$ does not happen once when $B'$ occurs $e\log_2 n -1$ times is at least
\begin{align}
\left(1-\frac{20e}{9n}\right)^{e\log_2 n -1} \ge \left(1-\frac{20e}{9n}\right)^{e\log_2 n} \ge 1-\frac{20e^2\log_2 n}{9n}.
\label{eq:pbpp}
\end{align}
}

{We note here that once $B'_{\le}$ happens, $a$ will increase by at most $\log_2 n -1$, and that once $B$ happens, $a$ will decrease by at least $1$. Then if $B$ or $B'$ occurs $n$ times (if possible for $a\ge n^c+1$) but $B'$ happens at most $e\log_2 n - 1$, we have $B$ happens at least $n-e\log_2 n +1$.  Further if only $B'_{\le}$ follows each occurrence of $B'$, then we know in total $a$ decreases by at least 
\begin{align*}
n-e&{}\log_2 n +1 - (e\log_2 n-1)(\log_2 n-1) = n-(e\log_2 n -1)\log_2 n \\
&\ge{} n-e(\log_2n)^2 \ge n-n^c,
\end{align*} 
where the last inequality holds for $n$ sufficiently large. That is, with these conditions, $a$ will drop below $n^c$. 
Via (\ref{eq:pBBp}), (\ref{eq:pBp}), and (\ref{eq:pbpp}), we know that the event that $a$ drops below $n^c$ before reaching the global optimum, happens with probability at least
\begin{align*}
\left(1-\frac{e}{n^{n^c+c-1}}\right)\left(1-\frac{1}{(en)^e}\right) \left(1-\frac{20e^2\log_2 n}{9n}\right) &\ge{} 1-\frac{e}{n^{n^c+c-1}}-\frac{1}{(en)^e}-\frac{20e^2\log_2 n}{9n} \\
\ge 1-\frac{17\log_2 n}{n},
\end{align*}
where the last inequality uses $e/n^{n^c+c-1} \le e/n$ for $n\ge 2^{1/c}$, $1/(en)^e \le 1/n$, and $(e+1)/\log_2 n+20e^2/9 \le 17$ for $n\ge 4^e$.}
\end{proof}

In simple words, 
this lemma holds trivially for the RLS where only one bit can be flipped for one generation. The main fact we use for the \oea's proof is that the probability of decreasing one zero is higher than the one of reaching the global optimum or the possible $(1,0)$ first bit pattern by a factor of $\Theta((1-b)^{n^{c}})$. Then with high probability the number of zeros decreases below $n^{c}$ before the global optimum or the possible $(1,0)$ pattern happens once. For the case when $(1,0)$ pattern occurs in the first generation, if Event I already happens, it will be stuck and the global optimum cannot be reached, otherwise, we consider the process afterwards and using the above consideration.

Note that the corresponding proofs in~\cite{ZhengCY21} of Theorem~\ref{thm:omt} for the \omt ($w=-n$) build the stagnation probability conditional on the global optimum cannot be reached before the individual decreases its number of zeros bits to $n^{1/3}$ and the initial $|X^1| > n/4$. Hence, taking $b=1/4$ and $c=1/3$ in Lemma~\ref{lem:startingassump}, and together with {Lemmata}~\ref{lem:stagneg} and~\ref{lem:nonzero}, we have the results for $w\in\Z_{\le -n}$ in the following. 
\begin{theorem}
Let {$n$ be sufficiently large} and $w\in \Z_{\le -n}$. Then for the $n$-dimensional \omw function, with a probability at least $1-(n+1)\exp{\left(-n^{1/3}/e\right)}-(e+1)/n^{1/3}-{17\log_2 n/n}$, RLS and \oea cannot reach the global optimum. 
\label{thm:lemn}
\end{theorem}
Note that the additional item of $-{17\log_2 n/n}$ corresponds to Lemma~\ref{lem:startingassump}.
 We note here that Lemma~\ref{lem:startingassump} can be added into~\cite[Lemmata~4-6 and Theorem~1]{ZhengCY21} for a more rigorous analysis. 

\subsection{Non-global-convergence When $w\in[-n+1..-1],n>0$}
\label{subsec:mnm1}
As mentioned above, from Lemma~\ref{lem:stagneg}, we know that for $w=-n$ (\omt discussed in~\cite{ZhengCY21}), Event I becomes that there is a $t_0 \in \N$ such that $(X_1^{t_0-1}, X_1^{t_0})=(0,1)$ and $X^{t_0}\ne 1^n$. It eases the theoretical analysis as one only needs to think about the event of the first bit changing from $0$ to $1$ regardless of other bits' situations before the number of zeros in the current solution is less than $2$. However, for $w\in[-n..-1]$, we need to consider not only the process of the first bit but also the process for the other bits, which is complicated than \omt.

{Now we briefly state that the analysis idea in~\cite{ZhengCY21} can not be directly applied here.} Starting from the first time if possible that the number of zeros decreased below $n^{c}$ discussed in Lemma~\ref{lem:startingassump}, now we consider four different initial types of the first bit pattern. Technically, we note the overall structural difference from the one in~\cite{ZhengCY21}. In~\cite{ZhengCY21}, the occurrence of $(0,1)$ first bit pattern before the global optimum directly results in the non-convergence. Then they built the non-convergence of the $(0,0)$ pattern by showing with high probability $(0,1)$ pattern will occur in the process afterwards, and built the non-convergence of the $(1,0)$ pattern by showing the easy transferring to the $(0,0)$. Then they showed the non-convergence of the $(1,1)$ pattern by transferring to $(1,0)$ and finally to $(0,1)$ or by staying the first bit of one and finally the occurrence of Event II. That is, the analyses for $(1,0),(0,0),$ and $(0,1)$ patterns are based on the occurrence of Event I, and $(1,1)$ pattern on the occurrence of either Event I or II.  
However, for our current case of $w\in[-n+1..-1]$, from Lemma~\ref{lem:stagneg}, $(0,1)$ is not the stagnation case for the \oea if other bits have $1$s less than $w+n-1$. That is, it is still possible to leave the $(0,1)$ pattern to $(1,0)$ or $(1,1)$ in some cases, which means that the analysis idea in~\cite{ZhengCY21} cannot be directly applied to our current case $w\in[-n+1..-1]$.

To tackle this, we first show that Event II already happens with a high probability starting from the $(1,1)$ first bit pattern if Event I doesn't happen in the considered starting point. Then other patterns can be transferred to this pattern and finally Event II happens or can be transferred to the occurrence of Event I. We note that the proof idea for the case of $(X^0_1,X^1_1)=(1,1)$ in~\cite[Lemma~6]{ZhengCY21} is to consider two situations, the first bit staying at $1$ before the number of zeros decreases below $n^c$, and the first bit once changing to $0$ before the number of zeros decreases below $n^c$. For the former situation, they calculate the probability of $1-1/n^{1-2c}-(n-1)e^{-n^c/e}$ that Event II happens afterwards, and for the latter situation, they calculate the probability that the first bit of $(0,1)$ will be reached (Event I for the \omt happens), then they then obtain the overall stagnation probability. Since now we consider the case of $(1,1)$ bit pattern for the first time the number of zeros decreased below $n^{c}$ discussed in Lemma~\ref{lem:startingassump}, we then directly extract their results for the former situation discussed above, and formalize in the following. 
\begin{lemma}
Let $c\in(0,1/2)$. Assume that the global optimum has not been reached before the number of zeros in the current individual drops below $n^{c}$, and let $g_0$ be the first generation that the number of zeros in the current individual drops below $n^{c}$ and $t_0$ be the corresponding decision time in Algorithm~\ref{alg:rlsoea}. If $(X^{t_0-1}_1,X^{t_0}_1)=(1,1)$, then there exists a $g' \in \N\cup \{0\} $, such that with a probability at least $1-1/n^{1-2c} -(n-1)e^{-n^c/e}$, after $g'$ generations Event II will happen. 
\label{lem:11}
\end{lemma}

Assume that the initial case is $(0,1)$ and Event I does not happen. For the \oea, if the number of zeros equals $-w$, then the only offspring that can enter the next generation is $1^n$, that is, Event II happens. Otherwise,  the $(1,1)$ happens with probability at least $1-e/n^{2-c}-1/n^{1-2c} -(n-1)e^{-n^c/e}$ conditional on the change of first bit pattern. See details in the proof of the following lemma.
\begin{lemma}
Consider the same assumption as in Lemma~\ref{lem:11}. If $(X^{t_0-1}_1,X^{t_0}_1)=(0,1)$, then there exists a $g' \in \N\cup \{0\} $, such that with probability at least $1-e/n^{2-c}-1/n^{1-2c} -(n-1)e^{-n^c/e}$, after $g'$ generations Event I or II will happen. 
\label{lem:01}
\end{lemma}
\begin{proof}
For RLS, if $w\in[-n..-2]$, then from Lemma~\ref{lem:stagneg}, Event I already happens, and thus this lemma trivially holds. In the following, we consider RLS for $w=-1$ and \oea for $w\in[-n..-1]$. 

Let $a$ be the number of zeros in $X^{t_0}$. If $a\le -w-1$, then we know that $|X^{t_0}_{[2..n]}| = n-a-1 \ge n+w$, that is, Event I happens. 

If $a=-w$, then $f(X^{t_0-1},X^{t_0}) = n+w$. For any generated offspring with the first bit value of $0$, it will have a fitness value at most $n-1+w$, which is less than its parent $X^{(g_0)}=X^{t_0}$, hence cannot enter into the next generation. The only case that the offspring can be accepted is $1^n$, which means that Event II happens. 

If $a\ge -w+1$, then for \oea, it is not difficult to see that the probability of changing the first bit pattern to $(1,0)$ is at most 
$$\binom{a}{-w+1}\frac{1}{n^{-w+1}}\frac1n,$$ 
and also not difficult to see that the probability of changing the first bit pattern to $(1,1)$ is at least 
$$\binom{a}{-w}\frac{1}{n^{-w}}\left(1-\frac1n\right)^{n+w}.$$
Hence, conditional on that the first bit pattern changes (that is, $(1,0)$ or $(1,1)$ happens), the probability of the first bit pattern changing to $(1,1)$ is at least
\begin{align*}
\frac{1}{1+\frac{\binom{a}{-w+1}\frac{1}{n^{-w+1}}\frac1n}{\binom{a}{-w}\frac{1}{n^{-w}}\left(1-\frac1n\right)^{n+w}}} 
&={}\frac{1}{1+\frac{(a+w)e}{(1-w)n^2}}\ge \frac{1}{1+\frac{(n^c+w)e}{(1-w)n^2}} \\
&\ge{} \frac{1}{1+\frac{(n^c-0)e}{(1-0)n^2}} = 1-\frac{n^ce}{n^ce+n^2}
 \ge 1-\frac{e}{n^{2-c}},
\end{align*}
where the first inequality uses $a\le n^c$, and the penultimate inequality uses $w\le -1<0$.
For RLS and $w=-1$, the first bit pattern can only change to $(1,1)$ because one zero bit in $X^{(g_0)}=X^{t_0}$ needs to be flipped to ensure $\tilde{X}^{(g_0)}$ has equal fitness to $X^{(g_0)}$ and thus enters into the next generation. Hence, the above lower bound of the conditional probability also holds.

We note that the process after the first bit pattern changes to $(1,1)$ turns to the case discussed in Lemma~\ref{lem:11}. Hence, we know the probability that Event II happens is at least
\begin{align*}
\left(1-\frac{e}{n^{2-c}}\right)&{}\left(1-\frac{1}{n^{1-2c}} -(n-1)e^{-\frac{n^c}{e}}\right) 
\ge 1-\frac{e}{n^{2-c}}-\frac{1}{n^{1-2c}} -(n-1)e^{-\frac{n^c}{e}}.
\end{align*}
Then it is proved.
\end{proof}

If the initial case is $(0,0)$, then via calculating the probability lower bound of $1-\frac{e}{n^{1-2c}}-\frac{1}{n^c}$ that the first bit flips to $1$ (that is, $(0,1)$ pattern happens) before the number of zeros in the current solution drops to $1$, we then turn to $(0,1)$ pattern in Lemma~\ref{lem:01}, and have the following lemma.
\begin{lemma}\label{lem:00}
Consider the same assumption as in Lemma~\ref{lem:11}. If $(X^{t_0-1}_1,X^{t_0}_1)=(0,0)$, then there exists a $g' \in \N\cup \{0\} $, such that with probability at least $1-e/n^{1-2c}-1/n^c -(n-1)e^{-n^c/e}$, after $g'$ generations Event I or Event II will happen. 
\end{lemma}
\begin{proof}
We discuss the process until there is $1$ zero or the first bit changes to $1$, conditional on that when the number of ones in the current individual changes, it only increases by $1$ (noting that this condition holds trivially for the RLS), which happens with probability at least
\begin{align}
\prod_{a=n^c}^2\left(1-\frac{ea}{n}\right) \ge \left(1-\frac{e}{n^{1-c}}\right)^{n^c} \ge 1-\frac{e}{n^{1-2c}},
\label{eq:con}
\end{align} 
where we use $a \le n^c$ and Lemma~\ref{lem:fact1} to obtain the above first expression. When the number of ones increases by $1$, by Lemma~\ref{lem:fact2}, we calculate the probability of the event that the first bit value stays at $0$ until there is $1$ zero in the current individual
\begin{align}
\prod_{a=n^c}^2\left(1-\frac 1a\right)=\frac{1}{n^c}.
\label{eq:0stay}
\end{align}
With (\ref{eq:con}) and (\ref{eq:0stay}), we know that the event that the first bit changes to $1$ before $a$ changes to $1$ happens with probability at least
\begin{align*}
\left(1-\frac{e}{n^{1-2c}}\right)\left(1-\frac{1}{n^c}\right) \ge 1-\frac{e}{n^{1-2c}}-\frac{1}{n^c}.
\end{align*}
Let $\tilde{g}$ be the generation for the first time the first bit changes to $1$ and let $\tilde{t}$ be the corresponding decision time. If $|X_{[2..n]}^{\tilde{t}}| \in [w+n..n-2]$, then Event I happens. Otherwise, the process afterwards turns to the case discussed in Lemma~\ref{lem:01}, and we know that Event I or II happens with probability at least $1-\frac{e}{n^{2-c}}-\frac{1}{n^{1-2c}} -ne^{-\frac{n^c}{e}}$. 
Hence, the overall probability for the current case that Event I or II happens is at least
\begin{align*}
\min&{}\left\{1-\frac{e}{n^{1-2c}}-\frac{1}{n^c}, 1-\frac{e}{n^{2-c}}-\frac{1}{n^{1-2c}} -(n-1)e^{-\frac{n^c}{e}}\right\}\\
&\ge{} 1-\frac{e}{n^{1-2c}}-\frac{1}{n^c} -(n-1)e^{-\frac{n^c}{e}},
\end{align*}
where the last inequality uses that for $n\ge 2$, $e/n^{1+c}+1\le e$, and thus
\begin{equation*}
\frac{e}{n^{2-c}}+\frac{1}{n^{1-2c}} = \frac{1}{n^{1-2c}} \left(\frac{e}{n^{1+c}}+1\right) \le \frac{e}{n^{1-2c}}.
\end{equation*}
Then it is proved.
\end{proof}

If the initial case is $(1,0)$, then via calculating the probability lower bound of $1-e/n^{1-c}$ that the first bit pattern turns to $(0,1)$ or $(0,0)$ and the number of zeros does not decrease, from {Lemmata}~\ref{lem:01} and~\ref{lem:00}, we have the following result.
\begin{lemma}
Consider the same assumption as in Lemma~\ref{lem:11}. If $(X^{t_0-1}_1,X^{t_0}_1)=(1,0)$, then there exists a $g' \in \N\cup \{0\} $, such that with probability at least $1-e/n^{1-2c}-1/n^c -(n-1)e^{-n^c/e}-e/n^{1-c}$, after $g'$ generations Event I or Event II will happen. 
\label{lem:10}
\end{lemma}
\begin{proof}
For the RLS, since only one bit can be flipped for each generation, we know that $X^{(g_0)}_{[2..n]} =X^{t_0}_{[2..n]} \ne 1^{n-1}$, and thus $\tilde{X}^{(g_0)}$ cannot be the global optimum. Hence, the first bit pattern afterwards turns to $(0,0)$ or $(0,1)$.

For the \oea, starting from $X^{(g_0)}_1=X^{t_0}_1=0$, we know that the probability of generating offspring with fewer zeros is at most $n^c/n=1/n^{1-c}$, and that the probability to generate an offspring that can enter into the next generation is at least $\left(1-1/n\right)^{n-a} \ge \left(1-1/n\right)^{n-1} \ge 1/e$ for $a$ the number of zeros in $X^{(g_0)}$(here we pessimistically consider generating the offspring with no $1$ from its parent changing to $0$, which will surely enter into the next generation). Hence, if the offspring enters into the next generation, then with probability at most $(1/n^{1-c})/(1/e)=e/n^{1-c}$, an offspring with fewer zeros can be reached. Hence, with probability of at least $1-e/n^{1-c}$, the global optimum cannot be reached and the first bit pattern turns to $(0,0)$ or $(0,1)$. 

Hence, from {Lemmata}~\ref{lem:01} and~\ref{lem:00}, we prove this lemma.
\end{proof}

Hence, noting the probability $\Pr[|X^1| \ge (3/4)n] \ge 1-\exp\left(-n/8\right)$, and noting that there are only four possible first bit pattern $(1,1), (0,1),(0,0)$, and $(1,0)$ for the first time the number of zeros drops below $n^{c}$, 
from {Lemmata}~\ref{lem:startingassump} to~\ref{lem:10}, we then obtain the probability of the RLS and \oea reaching the global optimum of the \omw.
\begin{theorem}
Let {$n$ be sufficiently large} and $w\in [-n..-1]$. Then for the $n$-dimensional \omw function, with a probability at least {$1-n\exp(-{n^{1/3}/e})-4/n^{1/3}$}, RLS and \oea cannot reach the global optimum. 
\label{thm:mnm1}
\end{theorem}
\begin{proof}
For the random initialization, we know $E[|X^1|]=n/2$. With the Chernoff inequality, we know
\begin{align}
\Pr\left[|X^1|\le bn\right] \le \exp\left(-\frac{(1-2b)^2}{2} n\right).
\label{eq:ini}
\end{align} 
Together with Lemma~\ref{lem:startingassump}, we know that the global optimum cannot be reached before the number of zeros in the current solution decreases below $n^{1/3}$ with probability at least
\begin{align*}
\left(1-\exp\left(-\frac{(1-2b)^2}{2} n\right)\right)\left(1-\frac{4e}{5}n\left(1-b\right)^{n^c}\right).
\end{align*}
Let $g_0$ be such first generation and $t_0$ be the corresponding decision time. We know that there are only four cases for the first bit pattern of $(X^{t_0-1}_1,X^{t_0}_1)$, $(0,0),(0,1),(1,0)$, and $(1,1)$. Therefore, from {Lemmata}~\ref{lem:11} to~\ref{lem:10}, we know that the probability that Event I or II happens is at least
\begin{align*}
\bigg(1-&{}\exp\left(-\frac{(1-2b)^2}{2} n\right)\bigg)\left(1-\frac{4e}{5}n\left(1-b\right)^{n^c}\right) 
\left(1-\frac{e}{n^{1-2c}}-\frac{1}{n^c} -\frac{n-1}{e^{\frac{n^c}{e}}}-\frac{e}{n^{1-c}}\right)\\
\ge{}&{} 1-\exp\left(-\frac{(1-2b)^2}{2} n\right) - \frac{4e}{5}n\left(1-b\right)^{n^c} 
-\frac{e}{n^{1-2c}}-\frac{1}{n^c} -(n-1)e^{-\frac{n^c}{e}}-\frac{e}{n^{1-c}}.
\end{align*}
Taking $b=1/4$ and $c=1/3$, we have the lower bound of the probability of the non-convergence to the global optimum as
\begin{align*}
1-&{}\exp\left(-\frac{1}{8} n\right)-{\frac{17\log_2 n}{n}} 
-\frac{e}{n^{1/3}}-\frac{1}{n^{1/3}} -(n-1)e^{-\frac{n^{1/3}}{e}}-\frac{e}{n^{2/3}}\\
&\ge{} 1-{n\exp\left(-\frac{n^{1/3}}{e}\right)}-\frac{4}{n^{1/3}},
\end{align*}
where we use {$-n/8<-n^{1/3}/e$ and $(17\log_2 n/n^{2/3})+(e/n^{1/3}) \le {3-e}$ for $n$ sufficiently large} for the inequality. 
\end{proof}
Note that $n\ge 88000$ can be relaxed to ensure the positive lower bound of the probability of the non-convergence to the global optimum by carefully tuning the $c$ and $b$ in the above proof. We will not conduct such tuning as currently we have already conveyed the information of the asymptotic $1-o(1)$ non-global-convergence probability, which is for $n$ sufficiently large.

From Theorems~\ref{thm:lemn} and~\ref{thm:mnm1}, we know that with any $w\in \Z_{<0}$, that is, once the current bit position has a different preference from its previous bit position, the RLS and \oea cannot reach the global optimum of the \omw with $1-o(1)$ probability.

\section{RLS and \oea on $\omw_{\in \Z_{\ge0}}$}
\label{sec:nonnegw}
Section~\ref{sec:negw} discussed the high probability of non-global-convergence of the RLS and \oea on the \omw with $w\in\Z_{<0}$. It agrees with the intuition. For the negative weight, the current and previous first bit values have different preferences, thus what has been learned for the current time step will be harmful when the time moves forward. Then it leads the algorithm to some local optimum. In this section, we will consider the case when the preference of the current and previous first bit agrees, that is, when $w\in \Z_{\ge 0}$.

\subsection{Global Optimum and Stagnation Case}
For \omw with $w=0$, the problem is the classic \om, the global optimum is $x^t=1^n$, and there is no stagnation case. For \omw with $w \in \Z_{> 0}$, the global optimum  is $(x^{t-1},x^t)=(1*,1^n)$ for any $*\in\{0,1\}^{n-1}$. For the RLS and \oea in Algorithm~\ref{alg:rlsoea}, we say the algorithm reaches the global optimum if there is a certain $t'$ such that $X^{t'}=1^n$ with stored $X^{t'-1}_1=1$. We note that it is different from the global optimum for $w\in \Z_{<0}$ in Section~\ref{sec:negw} where the stored $X^{t'-1}_1=0$.

We note that even when the preference of the current and previous first bit agrees, the RLS or \oea can also get stuck into some local optimum, see the following lemma.
\begin{lemma}
Let $w \in \Z_{> 0}$. Consider using the RLS / \oea to optimize the $n$-dimensional \omw function. Let $X^0, X^1, \dots$ denote the solution sequence. Let
\begin{itemize}
\item \textsl{Event III:} 
For the \oea: there is a $t_0 \in \N$ such that $(X_1^{t_0-1}, X_1^{t_0})=(1,0)$ and $|X_{[2..n]}^{t_0}| \in [n-w+1..n-1]$;
{for} the RLS: $w>1$, and there is a $t_0 \in \N$ such that $(X_1^{t_0-1}, X_1^{t_0})=(1,0)$.
\end{itemize}
If Event III happens at a certain time, then \oea / RLS cannot find the global optimum of the \omw in an arbitrary long runtime afterwards.
\label{lem:stagnonneg}
\end{lemma}
\begin{proof}
For the \oea, if Event III happens (say at generation $g_0$), then $X^{(g_0)}=X^{t_0}$ will have the fitness of at least $n+1$. However, from $X_1^{(g_0)}=0$ we know that any offspring $\tilde{X}^{(g_0)}$ will have the fitness value of at most $n < n+1$. Hence, $\tilde{X}^{(g_0)}$ cannot replace $X^{(g_0)}$, and then the stagnation happens.

For the RLS, if Event III happens (say at generation $g_0$), then with $w>1$ we know $\tilde{X}^{(g_0)}$ has the fitness of at most $|X^{(g_0)}|+1 < |X^{(g_0)}| + w$, that is, it has a fitness less than its parent $X^{(g_0)}$, and thus cannot enter into the next generation. Hence, the stagnation happens.
\end{proof}
The stagnation of Event III is from the fact that any generated offspring will have the first bit pattern of $(0,*)$ with $*\in\{0,1\}$, thus a lower fitness than its parent when the parent satisfies the condition in Event III, and therefore cannot replace its parent to the next generation.
%

\subsection{Global Convergence When $w\in \Z_{\ge 2}$}
This subsection discussed the case when $w\in \Z_{\ge 2}$, especially for the global convergence. We consider the different possible first bit patterns for $(X^{0}_1,X^{1}_1)$ in the first generation. Similar to Section~\ref{subsec:mnm1}, we first discuss the $(1,1)$ pattern. 
\begin{lemma}
Given a constant $b \in (0,0.5)$. Assume $|X^1| \ge bn$. If $(X^{0}_1,X^{1}_1)=(1,1)$, then 
\begin{itemize}
\item for the RLS, with probability $1$, the global optimum can be reached;
\item for the \oea, \textsl{(a)} with probability at least $e^{-(1-b)e}$, the global optimum can be reached before Event III happens; \textsl{(b)} with probability at least $e^{-(1-b)e}(1-e^3/n^{1/2})/{(1+(en+1)/\min\{n-|X^1|,w/2\})}$, Event III happens before the global optimum is reached.
\end{itemize}
\label{lem:11nonnegw}
\end{lemma}
\begin{proof}
Since the RLS only flips one bit each generation, we know that if the first bit flips to $0$, then the generated offspring $\tilde{X}^{(g)}$ has the fitness of $|\tilde{X}^{(g)}|=|X^{(g)}|-1< |X^{(g)}|+w,$ which is the fitness of $X^{(g)}$, and thus $\tilde{X}^{(g)}$ cannot enter into the next generation. That is, the first bit value stays at $1$ and eventually the global optimum is reached. The first part is proved.

Now we consider the \oea. Let $a$ denote the number of zeros in the current individual. Let the event $C$ denote that the generated offspring has more number of ones and its first bit value is $1$. Then we know 
$$\Pr[C] \ge \frac{a}{n} \left( 1-\frac{1}{n}\right)^{n-1} \ge \frac{a}{en}.$$ 
It is easy to see that such offspring will enter into the next generation due to the selection operator. On the other hand, let the event $D$ denote that the offspring with the first bit value of $0$ generated from a parent with first bit value of $1$, enters into the next generation, and we know $$\Pr[D] \le \frac{1}{n}\frac{a}{n}.$$
Hence, 
\begin{align*}
\frac{\Pr[C]}{\Pr[D]} \ge \frac{n}{e}, 
\end{align*}
and thus we know that the probability that $C$ happens $(1-b)n$ times before $D$ happens is at least
\begin{align}
\left(1-\frac{e}{n+e}\right)^{(1-b) n} = \left(1-\frac{1}{n/e+1}\right)^{\frac ne (1-b)e}  \ge \frac{1}{e^{(1-b)e}}.
\label{eq:lowercd}
\end{align}
%
%
Since $|X^1| \ge b n$, we know that the event that $C$ happens $(1-b) n$ times will result in reaching the global optimum, and thus the global convergence probability of the \oea is derived.

We now discuss the non-global-convergence part. Let $a_0$ be the number of zeros in $X^1$. If $a_0\le w-1$, then we consider the process afterwards. Otherwise, from the above analysis, we know that $1/e^{(1-b)e}$ in (\ref{eq:lowercd}) is also the lower bound of the probability that $C$ happens $(1-b)n-w+1$ times before $D$ happens once, and thus also the lower bound for the event that the first bit value stays at $1$ till the first time the number of zeros is less than $w$. We further show that the number of zeros is at least $w/2$ with a high probability the first time it drops below $w$. Let $a\ge w$ denote the number of zeros in a solution with the first bit value of $1$ before it drops below $w$, and conditional on its offspring with the first bit value of $1$, let $H$ denote the event that the offspring has the number of zeros less than $a$, and $G$ the event that the offspring has the number of zeros at least $a/2$. If $a=2$, from Lemma~\ref{lem:fact1}, we know
\begin{align}
\frac{\Pr[G]}{\Pr[H]} \ge 1-\frac{ae}{n-1}=1-\frac{2e}{n-1}.
\label{eq:gha2}
\end{align}
For $a>2$, we have
\begin{align*}
\Pr[H]\ge \frac{a}{n}\left(1-\frac1n\right)^{n-2} \ge \frac{a}{en},
\end{align*}
and
\begin{align*}
\Pr[G] \ge \Pr[H] - \frac{\binom{a}{a/2}}{n^{a/2}}.
\end{align*}
Hence,
\begin{align*}
\frac{\Pr[G]}{\Pr[H]} &\ge{} \frac{\Pr[H] - \frac{\binom{a}{a/2}}{n^{a/2}}}{\Pr[H]}=1-\frac{\frac{\binom{a}{a/2}}{n^{a/2}}}{\Pr[H]} 
\ge 1-\frac{\left(\frac{2e}{n}\right)^{a/2}}{\frac{a}{en}}
\ge 1-\frac{\left(\frac{2e}{n}\right)^{3/2}}{\frac{3}{en}} 
\ge 1-\frac{e^3}{n^{1/2}},
\end{align*}
where the second inequality uses $\binom{n}{k} \le (en/k)^k$, the third inequality uses $a\ge 3$, and the last inequality uses $2<e<3$. Thus together with (\ref{eq:gha2}), we have
\begin{align}
\frac{\Pr[G]}{\Pr[H]}  \ge \min\left\{1-\frac{2e}{n-1},1-\frac{e^3}{n^{1/2}}\right\}=1-\frac{e^3}{n^{1/2}}.
\label{eq:gh}
\end{align}
That is, with probability at least $1-\frac{e^3}{n^{1/2}}$, there are at least $w/2$ zeros in the starting solution when the number of zeros drops below $w$ for the first time. We now reuse $a$ be the number of zeros in the current solution, and consider $a\ge w/2$ in the following. Recall that $C$ is the event that the generated offspring has more number of ones and its first bit value is $1$, and $D$ the event that the offspring with the first bit value of $0$ generated from a parent with first bit value of $1$, enters into the next generation.
It is not difficult to see that
\begin{align*}
\Pr[C] \le \frac an
\end{align*}
and
\begin{align*}
\Pr[D] \ge \frac 1n \frac an \left(1-\frac 1n\right)^{n-2} \ge \frac{a}{en^2}.
\end{align*}
Hence, we have
\begin{align*}
\frac{\Pr[D]}{\Pr[C]} \ge \frac{1}{en}.
\end{align*}
Thus, we know that the probability that $D$ happens once before $C$ happens $\min\{a_0,w/2\}$ times is at least
\begin{align*}
1-\left(1-\frac{1}{en+1}\right)^{\min\{a_0,w/2\}} \ge \frac{1}{1+\frac{en+1}{\min\{a_0,w/2\}}},
\end{align*}
where the last inequality uses $1-(1-x)^n\ge 1/(1+1/(xm))$ for $x\in(0,1]$ and $m>0$, which is proven in~\cite[Lemma~2]{AntipovD21} and also in~\cite[Lemma~31]{DangL16}. 

Together with the above discussed probability of at least $1/e^{(1-b)e}$ for this condition that the first bit value stays at $1$ till the first time the number of zeros is less than $w$, and (\ref{eq:gh}), the non-global-convergence probability for the \oea is derived.
\end{proof}

Due to the random initialization for $X^0$ and $X^1$, Lemma~\ref{lem:11nonnegw} has already shown that with at least a constant probability RLS and the \oea reach the global optimum, and with at least $\Theta(\min\{n-|X^1|,w/2\}/(en+1))$ probability, the \oea cannot reach the global optimum, which is already the information we plan to convey. In order to see more information, like whether the RLS can reach the global optimum with $1-o(1)$ probability conditional on that the algorithm does not get stuck for the first generation, we will still consider other initial first bit patterns in the following.

Lemma~\ref{lem:01nonnegw} collects the results if the initial first bit pattern is $(0,1)$ for the RLS and the \oea. 
\begin{lemma}
Assume $|X^1| \ge bn$. If $(X^{0}_1,X^{1}_1)=(0,1)$, then 
\begin{itemize}
\item for the RLS, with probability $1-1/n$, the global optimum can be reached;
\item for the \oea,
\textsl{(a)} with probability at least $\left(1-1/n\right) e^{-(1-b)e-1}$, the global optimum can be reached before Event III happens;
\textsl{(b)} with probability at least $(1-1/n) e^{-(1-b)e-1}(1-e^3/n^{1/2})/(1+{(en+1)}/\min\{n-|X^1|,w/2\})$ for $(0,1)$, Event III happens before the global optimum is reached.
\end{itemize}
\label{lem:01nonnegw}
\end{lemma}
\begin{proof}
For the RLS, from $(X^{0}_1,X^{1}_1)=(0,1)$, we know that with probability of $1-1/n$, the generated $\tilde{X}^{(g)}$ has its first bit value of $1$, which has better fitness than its parent $X^{(g)}$ and surely enters into the next generation. Then it turns to the $(1,1)$ case. Together with Lemma~\ref{lem:11nonnegw}, the first part is proved.

Now we consider the \oea. We know that with probability of at least $$\left(1-\frac 1n\right)^n\ge \left(1-\frac1n\right)\frac1e,$$ the generated $\tilde{X}^{(g)}$ has at least the same number of ones as its parent $X^{(g)}$ and  has its first bit value of $1$. Due to definition, we know $f(\tilde{X}^{(g)}) \ge f(X^{(g)}) + w \ge f(X^{(g)})$, thus it will enter into the next generation, that is, $X^{(g+1)}=\tilde{X}^{(g)}$, and thus $X^{(g+1)}_1=\tilde{X}^{(g)}_1=1$. Hence, together with Lemma~\ref{lem:11nonnegw}, 
the \oea part in this lemma is proved.
\end{proof}

If the initial first bit pattern is $(0,0)$, considering the global convergence probability, since the $(0,0)$ first bit pattern cannot be the stagnation case, the first bit will eventually flip to $1$, and turn into the above $(0,1)$ case. We could directly apply the above discussed result for the lower bound of the global convergence probability. However, recalling that for the lower bound of the non-global-convergence probability, in Lemma~\ref{lem:11nonnegw} we have $n-|X^1|$, which is the number of zeros in the solution the first time the first bit pattern becomes $(1,1)$, we cannot directly reuse the way for the global convergence probability. As discussed previously, Lemma~\ref{lem:11nonnegw} has already conveyed our main information, and thus we omit the non-global-convergence result here. See the following lemma.
\begin{lemma}
Assume $|X^1| \ge bn$. If $(X^{0}_1,X^{1}_1)=(0,0)$, then 
\begin{itemize}
\item for the RLS, with probability $1-1/n$, the global optimum can be reached;
\item for the \oea, with probability at least $\left(1-1/n\right) e^{-(1-b)e-1}$, the global optimum can be reached before Event III happens.
\end{itemize}
\label{lem:00nonnegw}
\end{lemma}

For $(1,0)$ initial first bit pattern, Event III happens for the RLS. For the \oea, if Event III does not happen in the first generation, the offspring with $w$ more ones than its parent $X^1$ can still enter into the next generation. It happens with the probability of at most $\binom{a}{w}\frac{1}{n^w}$ with $a$ the number of zeros in $X^1$ and results in a runtime lower bound of $n^w/\binom{a}{w}$. Since this lemma heavily depends on the relationship between $w$ and $|X^{1}|$ and the information provided for our main message is limited, it will not be considered in Theorem~\ref{thm:0inf}, but for a complete picture, we still list its results, see the following lemma.
\begin{lemma}
Assume $|X^1| \ge  bn$. Let $k=|X^1_{[2..n]}|$. If $(X^{0}_1,X^{1}_1)=(1,0)$, then
\begin{itemize}
\item for the RLS, Event III happens;
\item for the \oea,
\textsl{(a)} if further $k \ge n-w+1$, then Event III happens;
\textsl{(b)} otherwise, with probability at least $\left(1-1/n\right) e^{-(1-b)e-1}$, the global optimum can be reached before Event III happens; Moreover, the expected runtime to reach the global optimum is at least $\left({w}/{((1-b)e)}\right)^w$.
\end{itemize}
\label{lem:10nonnegw}
\end{lemma}
\begin{proof}
The RLS part and the first part for the \oea in this lemma are trivial from the definition of Event III.

For the second part of the \oea, since Event III doesn't happen in the first generation, we only need to generate an offspring with at least $w+|X^1_{[2..n]}|$ number of ones, which is possible, and it will have the same or better fitness compared with $f(X^{0},X^{1}) = w+|X^1_{[2..n]}|$, and thus surely enter into the next generation. Once new offspring enters into the next generation, the first bit pattern becomes either $(0,1)$ or $(0,0)$. Then together with {Lemmata}~\ref{lem:01nonnegw} and~\ref{lem:00nonnegw}, we prove the convergence results.

Moreover, it is not difficult to see that generating an offspring to leave the $(1,0)$ pattern happens with probability at most
$$\frac{\binom{n-k+1}{w}}{n^w} \le \left(\frac{e(n-k+1)}{nw}\right)^w,$$ 
thus we need at least 
$$\left(\frac{nw}{e(n-k+1)}\right)^w \ge \left(\frac{1}{(1-b)e}w\right)^w$$ expected iterations. 
\end{proof}

Considering the $(1,1),(0,1),$ and $(0,0)$ initial first bit patterns from {Lemmata}~\ref{lem:11nonnegw} to~\ref{lem:00nonnegw}, we then have the general result for reaching the global optimum in the following theorem. 
\begin{theorem}
Let $w\in \Z_{\ge 2}$. Then for the $n$-dimensional \omw function, 
\begin{itemize}
\item for the RLS, \textsl{(a)} if $(X_1^0,X_1^1)\ne (1,0)$, which happens with probability $3/4$, then with probability $1-1/n$, the global optimum will be reached. Moreover, conditional on $(X_1^0,X_1^1)\ne (1,0)$ the expected runtime is $\Theta(n\log n)$;
\textsl{(b)} if $(X_1^0,X_1^1)= (1,0)$, which happens with probability $1/4$, the global optimum cannot be reached in an arbitrarily long time;
\item for the \oea, let $n\ge 500$, then
\textsl{(a)} the global optimum can be reached with probability at least $(1/{(4e^{\frac34e+1})}) (1-2\exp(-{n}/{24}))(e+2-2/n)$. Moreover, conditional on the event that starting from $|X^1| \ge n/4$ and $(X_1^0,X_1^1)\ne (1,0)$ and the first bit value stays at the value of $1$ once it is reached, which happens with at least such probability, the expected runtime is $\Theta(n\log n)$; 
\textsl{(b)} the global optimum can be reached with probability at most $1-(1-22/n^{1/2})({(2e+1)}/{8})(e^{-3e/4-1}/(1+4(en+1)/{\min\{n,2w\}}))$.
\end{itemize}
\label{thm:0inf}
\end{theorem}
\begin{proof}
The results for the RLS is directly from {Lemmata}~\ref{lem:11nonnegw} to~\ref{lem:00nonnegw}. For the expected runtime conditional on $(X^0_1,X^1_1)\ne (1,0)$, a simple coupon collector process then results in $\Theta(n \log n)$ expected runtime.

Now we consider the \oea. For the random initialization, we know that $E[|X^1|] = n/2$. With Chernoff bound~\cite[Corollary~1.10.6]{Doerr20}, we have
\begin{equation}
\begin{split}
\Pr&{}\big[|X^1|\in \left[b n.. \left(1-b\right) n\right]\big] 
\ge 1-2\exp\left(-\frac {\left(1 -2b\right)^2}{3} \frac n2\right) \\
&={} 1-2\exp\left(-\frac {\left(1 -2b\right)^2n}{6}\right).
\end{split}
\label{eq:ini}
\end{equation}
Since there are only four possible cases for $(X^{0}_1,X^1_1)$ with equal probability of $1/4$, $(0,0),(0,1),(1,0),$ and $(1,1)$, from {Lemmata}~\ref{lem:11nonnegw} to~\ref{lem:00nonnegw}, we know that the probability for the algorithm to reach the global optimum of the \omw function is at least
\begin{align*}
\Bigg(1-2\exp\left(-\frac {\left(1-2b\right)^2n}{6}\right)\Bigg)&\frac14\Bigg(e^{-(1-b)e}  
+\left(1-\frac 1n\right) e^{-(1-b)e-1} + \left(1-\frac 1n\right) e^{-(1-b)e-1}\Bigg)\\
={}&{} \frac{1}{4e^{(1-b)e+1}} \left(1-2\exp\left(-\frac {\left(1-2b\right)^2n}{6}\right)\right)
\left(e+2-\frac2n\right).
\end{align*}
Taking $b=1/4$, we have derived the global convergence probability for the \oea.

Essentially, the above probability for the global convergence is for the process that the first bit value stays at $1$ once it is generated and the global optimum is eventually reached. Consider the process conditional on the above event. It is not difficult to see that before the first bit reaches the value of $1$, the conditional process is identical to the original one, and after that the conditional process is identical to the one of the \oea optimizing the $(n-1)$-dimensional \om function. Hence, the runtime for the conditional process is $\Theta(n\log n)$. 

For the non-global-convergence result, we consider the cases of $(X^0_1,X^1_1)=(1,1)$ and $(0,1)$ from {Lemmata}~\ref{lem:11nonnegw} and~\ref{lem:01nonnegw}, and with (\ref{eq:ini}) for $|X^1|\le (1-b)n$,  the non-global convergence probability is at least
\begin{align*}
\Bigg(1-2&\exp\left(-\frac {\left(1-2b\right)^2n}{6}\right)\Bigg)\frac14\Bigg(\frac{e^{-(1-b)e}(1-e^3/n^{1/2})}{1+\frac{en+1}{\min\{n-|X^1|,w/2\}}}
+\left(1-\frac 1n\right) \frac{e^{-(1-b)e-1}(1-e^3/n^{1/2})}{1+\frac{en+1}{\min\{n-|X^1|,w/2\}}}\Bigg)\\
={}&{}\frac14\left(1-2\exp\left(-\frac {\left(1-2b\right)^2n}{6}\right)\right)
\cdot\left(1+\frac{1}{e}\left(1-\frac1n\right)\right)\frac{e^{-(1-b)e}(1-e^3/n^{1/2})}{1+\frac{en+1}{\min\{n-|X^1|,w/2\}}}\\
\ge{}&{} \left(1-2\exp\left(-\frac {\left(1-2b\right)^2n}{6}\right)-\frac{e^3}{n^{1/2}}\right)
\cdot \frac{\left(1+\frac{1}{e}\left(1-\frac12\right)\right)e^{-(1-b)e}}{4\left(1+\frac{en+1}{\min\{bn,w/2\}}\right)}\\
\ge{}&{} \left(1-\frac{22}{n^{1/2}}\right)\frac{\frac{2e+1}{2e}e^{-(1-b)e}}{4\left(1+\frac{en+1}{\min\{bn,w/2\}}\right)}
=\left(1-\frac{22}{n^{1/2}}\right)\frac{\frac{2e+1}{8}e^{-(1-b)e-1}}{1+\frac{en+1}{\min\{bn,w/2\}}},
\end{align*}
where the first inequality uses $|X^1|\le (1/2 + b)n$ and $n\ge 2$, and the last inequality uses $2\exp(-(1-2b)^2n/6) \le (22-e^3)/n^{1/2}$ for constant $b$ and sufficiently large $n$. Taking $b=1/4$ and $n\ge500$ the theorem is proved.
\end{proof}
From Theorem~\ref{thm:0inf}, we know that for RLS and the \oea with at least a constant probability, the process will converge to the global optimum. For the RLS, if the random initialization does not result in stagnation, it can reach the global optimum in $1-o(1)$ probability. However, both algorithms have the possibility of getting stuck. The random initialization then result in a probability of $1/4$ for the RLS to get stuck, and only $1/n$ stagnation probability if not stuck in the first generation. For the \oea, other than the stagnation case in the random initialization, it can get stuck with probability at least $\min\{\Theta(w/n),c\}$ for $c$ some positive constant less than $1$, and we note that this probability can be at least a constant when $w = \Omega(n)$. 

\subsection{Global Convergence When $w\in\{0,1\}$}
Now we consider the remaining cases $w=\{0,1\}$. For $w=0$, the problem is the classic \om function. For this function, the convergence is trivial for RLS and \oea, and the runtime results are known~\cite{Doerr20,Witt13}. 

For $w=1$, we see
from Lemma~\ref{lem:stagnonneg} that Event III excludes the case of $w=1$ for RLS. Actually, for the \oea, it indeed requires $n-w+1\le n-1$, that is, $w\ge 2$ to ensure the existence of $[n-w+1..n-1]$. Now we briefly discuss the behavior of RLS and \oea. Our statement for the convergence follows from the fact that the first bit pattern can transfer to other patterns before the global optimum is reached. In more detail, the first bit pattern $(*,1)$ with any $*\in\{0,1\}$ can transfer to $(1,1)$ via flipping one zero bit of the parent and keeping other bits unchanged, which can be naturally achieved by the RLS and \oea. $(*,0)$ with any $*\in\{0,1\}$ can change to $(0,1)$ via flipping the first bit of the parent and keeping other bits unchanged. $(1,1)$ can stay via flipping one zero bit of the parent and keeping other bits unchanged, and eventually the process reaches the global optimum. 


In terms of the expected runtime, we consider the conditional process that the first bit value stays at $1$ once it is reached from any initial first bit pattern except $(1,0)$, which happens with probability of at least a constant. Under this condition, with similar statements in the proof of Theorem~\ref{thm:0inf}, we have the expected runtime of $\Theta(n\ln n)$ for both RLS and \oea~\cite{Doerr20,Witt13}.
\begin{theorem}
Consider using the RLS and \oea to optimize the \omw function.
\begin{itemize}
\item For $w=0$, both algorithms can reach the global optimum with probability of $1$, and the expected runtime is $\Theta(n\ln n)$~\cite{Doerr20,Witt13}.
\item For $w=1$, both algorithms can reach the global optimum with probability $1$. Conditional on the event conditional on the event that starting from $|X^1| \ge n/4$ and $(X_1^0,X_1^1)\ne (1,0)$ and the first bit value stays at the value of $1$ once it is reached, which happens with probability at least $(1/{(4e^{\frac34e+1})}) \left(1-2\exp\left(-{n}/{24} \right)\right)\left(e+2-2/n\right)$, the expected runtime is $\Theta(n\ln n)$.
\end{itemize}
\label{thm:01}
\end{theorem}
From Theorems~\ref{thm:0inf} and~\ref{thm:01}, we see that for $w\in\Z_{\ge0}$, RLS and \oea reach the global optimum of the \omw function with at least a constant probability, and conditional on an event happening with at least a constant probability, the expected runtime is $\Theta(n\ln n)$, which is the same asymptotic complexity for the \om without the time-linkage property.

\section{Discussions}
\label{sec:exp}
\subsection{Understanding the Theoretical Results}
Sections~\ref{sec:negw} and~\ref{sec:nonnegw} separately show the behaviors of the RLS and \oea on the \omw with general weight $w\in\Z$. Here we will discuss them together in a more intuitive way. The evaluation of the time-linkage function relies on the current and historical solutions (say $N_{\text{his}}$ historical ones). Then if each solution has $n$ dimensions, we know that the current and the historical solutions together are in the space with dimension size $n(N_{\text{his}}+1)$. Since the historical solutions have already existed, we only optimize the current one, which is in $n$-dimensional space, while we expect a good outcome, which lies in the $n(N_{\text{his}}+1)$-dimensional space. The search space smaller than the aiming space might result in some stagnation cases. 

\omw in this paper follows \omt that only considers the first bit value for the time-linkage effect, and discusses the different weights of the last step to influence the function value. For the RLS and \oea optimizing \omw, we aim at the global optimum in $(n+1)$-dimensional space and search only in $n$-dimensional space. If $w\in\N_{<0}$, for a better function value, the first bit in the current step prefers a value of $1$, while we prefer $0$ for a stored first bit value of the last step. 
Hence, the first bit pattern of $(0,1)$ will become difficult to jump out as the generated offspring has the stored first bit value of $1$ and thus needs more gains in the current step to defeat the advantage of its parent with stored first bit value of $0$. Besides, the first bit pattern of $(1,1)$ will be more likely to stay since if the generated offspring changes its first bit to $0$ then it already results in a worse fitness by $w$ in terms of the first bit, and must need more gains in the current step to defeat its parent possibly. 
Theorem~\ref{thm:mnm1} shows that for any $w\in\N_{<0}$, even the largest $w=-1$, with a high probability, both algorithms will get stuck in a local optimum in the subspace $\{0\}\times\{0,1\}^n$ when there are no possible gains in the current step to leave, or stays with $(1,1)$ first bit pattern and eventually move to the local optimum $(1,1^n)$ in the subspace $\{1\}\times\{0,1\}^n$. 

If $w\in\N_{\ge 0}$, although the first bits in both current and previous steps prefer the value of $1$, there are stagnation situations for the $(1,0)$ first bit pattern and $w\ne\{0,1\}$. Note that any generated offspring has the stored first bit of $0$, which means a function value loss of $w$ in terms of the first bit. Thus it needs to gain more than $w$ in the current step to defeat its parent. However, such gains cannot always be satisfied. On the positive side, the first bit pattern of $(1,1)$ is more likely to stay and the process eventually moves to the global optimum. Theorem~\ref{thm:0inf} shows the non-convergence result and at least a constant probability of reaching the global optimum.

\subsection{A Primary Discussion on \mpoea}\label{subsec:mup1}
Before concluding this work, we give an additional discussion on how \mpoea optimizes the \omw function. \mpoea is similar to the \oea and the only difference is that there are $\mu$ parent individuals and each parent individual has its own stored historical solution. In each generation, one parent is picked uniformly at random to generate offspring by standard bit-wise mutation. One individual in the combined parent and offspring population with the worst fitness will be removed (ties are broken randomly). For \mpoea optimizing the \omw function, we consider the offline mode as in~\cite{ZhengCY21}. Only the simple case of $w \le -n$ is briefly discussed and we leave the analysis for $w > -n$ as our future work{. We conjecture the good performance of the \mpoea as we guess that the individuals with the stagnation cases (Events I to III) will not overwhelmingly take over the population due to their fitnesses}. For $w\le -n<0$ (obviously including $w=-n$, which is the \omt discussed in~\cite{ZhengCY21}), the optimum is reached if there is at least one individual $X^{g'}$ in the population $P^{g'}$ for a certain generation $g'$ such that $X^{g'}=1^n$ and its stored $X^{g'-1}_1=1$. Besides, when a population is given, the fitness rankings of the individuals are the same for all $w\le -n$, shown in the following lemma.
\begin{lemma}
Let $f_w$ denote the \omw function and $M \in \N$ be a constant. For any $P'=\{X_1,\dots,X_M \mid X_i=(X_{i,1},\dots,X_{i,n})\in\{0,1\}^n\}$ and $P''=\{Y_1,\dots,Y_M \mid Y_i=(Y_{i,1},\dots,Y_{i,n})\in\{0,1\}^n\}$, let $f_w(P',P'')=\{f_w(x_1,y_1),\dots,f_w(x_M,y_M)\}$ and let $r_{i,w}$ be the rank of $f_w(x_i,y_i)$ in $f_w(P',P'')$ (the ones with the same fitness share the same rank), then $r_{i,w_1}=r_{i,w_2}$ for any two $w_1,w_2 \le -n$. 
\label{lem:iden}
\end{lemma}
\begin{proof}
The key in this proof is to show that for any two $w_1\le w_2\in\Z_{\le -n}$ and any $i,j\in\{1,\dots,M\}$, $f_{w_1}(X_i,Y_i) \le f_{w_1}(X_j,Y_j) \Leftrightarrow f_{w_2}(X_i,Y_i) \le f_{w_2}(X_j,Y_j)$. Note that
\begin{equation}
\begin{split}
& f_{w}(X_i,Y_i) \le f_{w}(X_j,Y_j) \\
 \Leftrightarrow {} &wX_{i,1}+\sum_{k=1}^nY_{i,k} \le wX_{j,1}+\sum_{k=1}^nY_{j,k} \\
 \Leftrightarrow {} &w(X_{i,1}-X_{j,1}) \le \sum_{k=1}^nY_{j,k} - \sum_{k=1}^nY_{i,k}.
 \end{split}
 \label{eq:eqw}
\end{equation}
Obviously, if $f_{w}(X_i,Y_i) \le f_{w}(X_j,Y_j)$, then $X_{i,1}\ge X_{j,1}$. Otherwise, since $w\le -n$, we have $w(X_{i,1}-X_{j,1}) > -w \ge n \ge \sum_{k=1}^nY_{j,k}$, which contradicts to $w(X_{i,1}-X_{j,1}) \le \sum_{k=1}^nY_{j,k} - \sum_{k=1}^nY_{i,k}$. Hence, the event that $X_{i,1}=0,X_{j,1}=1$ cannot happen for $f_{w}(X_i,Y_i) \le f_{w}(X_j,Y_j)$.

If $X_{i,1}=X_{j,1}$, then if $f_{w_1}(X_i,Y_i) \le f_{w_1}(X_j,Y_j)$, we have $\sum_{k=1}^nY_{j,k} - \sum_{k=1}^nY_{i,k} \ge 0=w_2(X_{i,1}-X_{j,1})$, that is, $ f_{w_2}(X_i,Y_i) \le f_{w_2}(X_j,Y_j)$. Vice versa.

If $X_{i,1}=1,X_{j,1}=0$, (\ref{eq:eqw}) is equivalent to 
\begin{align*}
w \le \sum_{k=1}^nY_{j,k} - \sum_{k=1}^nY_{i,k}.
\end{align*}
Since $\sum_{k=1}^nY_{j,k} - \sum_{k=1}^nY_{i,k} \ge - \sum_{k=1}^nY_{i,k} \ge -n$, we know that the above inequality holds for all $w\le -n$. 

Therefore, this lemma is proved.
\end{proof}

Lemma~\ref{lem:iden} indicates that for all $w\in\Z_{\le -n}$ the stochastic optimization process of the \mpoea are identical, that is, to the one for \omt ($w=-n$), hence, the result in~\cite{ZhengCY21} also holds for $w\in\Z_{\le -n}$, shown in the following theorem.

\begin{theorem}
Consider using the \mpoea with $\mu=cn$ with sufficient large constant $c>0$ to optimize the \omw function with $w \le -n$. Then with $1-o(1)$ probability, the global optimum can be reached. Conditional on an event that happens with $1-o(1)$ probability, the global optimum can be reached in $O(\mu n)$ expected fitness evaluations.
\end{theorem} 

\section{Conclusion and Future Work}
\label{sec:con}
This work generalized the extreme weight of the time-linkage first bit in the only time-linkage theoretical benchmark \omt~\cite{ZhengCY21}, and analyzed the behaviors of the RLS and \oea on this generalized time-linkage benchmark function. We proved that except for the weights of $0$ and $1$ for which the RLS and \oea find the global optimum with probability $1$, with a positive probability they cannot converge to the global optimum{, that is, the time-linkage property generally makes the \om a harder problem}. Moreover, when the time-linkage weight is negative, neither algorithm can reach the global optimum with $1-o(1)$ probability. When the time-linkage weight is non-negative, the non-global-convergence probability is at least $\min\{\Theta(w/n),c\}$ for a certain $c\in(0,1)$, but both algorithms can reach the global optimum with at least a constant probability.  

With the insight gained on the problem, we could conjecture the influence of the time-linkage strength for general optimization problems (if they adopt the weight in a similar linear form). With a negative weight, the time-linkage dimension values in the current solution and historical solutions have different search biases (directions). Since the algorithms only optimize the current search space, two local optimum cases (for the whole space consisting of the current space and the time-linkage historical space) can occur. One is that the current solution is eventually optimized in the current space but with the stored value in the wrong search direction. The other is that the time-linkage dimension reaches the right search direction both in the current and historical solutions but no improvement can happen for its offspring. The second case is possible because the right direction in the current parent will be the wrong direction as the historical value for its offspring, and the fitness loss caused by the wrong historical value cannot be overweighted by the gain for the offspring. When any of two cases happens, the search process will move to the local optimum and cannot escape. The non-global-convergence happens. We conjecture that such a non-global-convergence probability is quite high. 

With the positive weight, the time-linkage dimension values in the current solution and the historical solution have the same search direction, which is beneficial for the global convergence. However, due to the unbiased mutation operator, the time-linkage dimension in the current solution with a value opposite to the search direction can be generated. The local optimum can be reached in the subspace restricted by the time-linkage dimension with its historical value along the search direction but its current value against the search direction. If further such a local optimum has better fitness than its parent, then it survives to the next generation. After that any offspring of such individual will store a wrong direction value in the time-linkage dimension, and no gain in the current time can overweight the loss of storing such wrong value. Then the global optimum cannot be reached further. We conjecture that such non-global-convergence probability is not high. 

For future work, we will analyze the behavior of other evolutionary algorithms on the \omw { (We conjecture the relatively good performance of the non-elitist algorithms as the non-elitism allows a chance  of escaping the local optima as in~\cite{ZhengZCY21})}. Besides, the current analyses are for the offline mode in the language of~\cite{ZhengCY21}, and it is more interesting to discuss the online mode, which may be more common in real-world applications. In addition,
the tools for proving the non-convergence are mainly from the elementary analysis. Similar to the advanced tools such as~\cite{HeY01,HeY03,DoerrJW12,CorusDEL17} for the runtime analysis, it will be interesting to develop some advanced tools for calculating the probability of convergence and the runtime conditional on the convergence when the problems are complicated.

\section*{Acknowlegements}
This work was supported by Science, Technology and Innovation Commission of Shenzhen Municipality (Grant No. GXWD20220818191018001), Guangdong Basic and Applied Basic Research Foundation (Grant No. 2019A1515110177), Guangdong Provincial Key Laboratory (Grant No. 2020B121201001), the Program for Guangdong Introducing Innovative and Enterpreneurial Teams (Grant No. 2017ZT07X386), Shenzhen Science and Technology Program (Grant No. KQTD2016112514355531).
%
\newcommand{\etalchar}[1]{$^{#1}$}

\end{document}